%% file: spca_bounds.tex
\documentclass[english,11pt]{article}
\usepackage[latin9]{inputenc}
\usepackage{color}
\usepackage{amsthm}
\usepackage{amsmath}
\usepackage{amssymb}
\usepackage{esint}

\newif\ifFULL
\FULLtrue

\makeatletter
\input{preamble}
\input{macros}

\makeatother

\begin{document}

\title{On the Worst-Case Approximability of Sparse PCA}

\author{Siu On Chan\thanks{Chinese University of Hong Kong. sochan@gmail.com.
  Most of this work was done while the author was a postdoc at Microsoft
Research New England.} \and
Dimitris Papailliopoulos\thanks{UC Berkeley. dimitrisp@berkeley.edu. 
Supported by NSF awards CCF-1217058 and CCF-1116404, and MURI AFOSR grant 556016.} \and
Aviad Rubinstein\thanks{UC Berkeley. aviad@eecs.berkeley.edu. 
Parts of this work were done at the Simons Institute for the Theory of Computing and Microsoft Research New England.
Supported by NSF grants CCF0964033 and CCF1408635, and by Templeton Foundation grant 3966. } 
}
\maketitle
\begin{abstract}
It is well known that Sparse PCA (Sparse Principal Component Analysis) is \NP-hard to solve exactly on worst-case instances.
What is the complexity of solving Sparse PCA approximately?
Our contributions include:
\begin{enumerate}
\item a simple and efficient algorithm that achieves an $n^{-1/3}$-approximation;
\item \NP-hardness of approximation to within $(1-\varepsilon)$, for some small constant $\varepsilon > 0$;
\item \SSE{}-hardness of approximation to within {\em any} constant factor; and 
\item an $\exp\exp\left(\Omega\left(\sqrt{\log \log n}\right)\right)$ (``quasi-quasi-polynomial'') gap for the standard semidefinite program.
\end{enumerate}

\end{abstract}

\input{intro}

\input{n_to_one-third.tex}

\input{NP-hardness.tex}

\input{sse-hardness.tex}
\input{rank-gap.tex}

\ifFULL
\input{aPTAS.tex}
\input{non-PSD.tex}
\fi

\begin{acknowledgement*}
We would like to thank Robert Krauthgamer, Tsz Chiu Kwok, Lap Chi Lau, Prasad Raghavendra, and Aaron Schild for numerous suggestions and helpful discussions.
\end{acknowledgement*}
\bibliographystyle{plain}
\bibliography{spca_bounds}

\end{document}

%% file: preamble.tex
\newcommand{\ignore}[1]{}%

\usepackage{complexity}

\usepackage{nameref}
\usepackage{bbm}
\usepackage{url}

\makeatletter
\let\orgdescriptionlabel\descriptionlabel
\renewcommand*{\descriptionlabel}[1]{%
  \let\orglabel\label
  \let\label\@gobble
  \phantomsection
  \edef\@currentlabel{#1}%
  \let\label\orglabel
  \orgdescriptionlabel{#1}%
}
\makeatother

\newcommand{\val}{\textsc{Val}}

\usepackage{amsmath,amssymb,amsthm,amsfonts,latexsym,bbm,xspace,graphicx,float,mathtools}
\usepackage[backref,colorlinks,citecolor=blue,bookmarks=true]{hyperref}
\usepackage[letterpaper,margin=1in]{geometry}
\usepackage{color}
\usepackage{algorithm}
\usepackage{algorithmic}

\usepackage{geometry} 
\usepackage{graphicx} 

\usepackage{booktabs} 
\usepackage{array} 
\usepackage{paralist} 
\usepackage{verbatim} 
\usepackage{subfig} 

\usepackage{fancyhdr} 
\usepackage{sectsty}
\allsectionsfont{\sffamily\mdseries\upshape} 

\usepackage[nottoc,notlof,notlot]{tocbibind} 
\usepackage[titles,subfigure]{tocloft} 

\usepackage{thmtools}
\usepackage{thm-restate}
\ignore{

}

\newtheoremstyle{named}{}{}{\itshape}{}{\bfseries}{.}{.5em}{#1 #3}
\theoremstyle{named}

\DeclareMathOperator{\supp}{supp}

\theoremstyle{remark}
\newtheorem*{acknowledgement*}{\protect\acknowledgementname}

\usepackage{babel}

  \providecommand{\acknowledgementname}{Acknowledgement}

\newcommand{\vecx}{\mathbf{x}}

\newcommand{\vecw}{\mathbf{w}}

\newcommand{\matA}{\mathbf{A}}
\newcommand{\matB}{\mathbf{B}}
\newcommand{\matC}{\mathbf{C}}
\newcommand{\matI}{\mathbf{I}}
\newcommand{\matJ}{\mathbf{J}}

\newcommand{\matX}{\mathbf{X}}

\newcommand{\SPCA}{\textsf{SparsePCA}}
\newcommand{\DkS}{\textsf{DkS}}
\newcommand{\SSE}{\textsf{SSE}}
\newcommand{\MC}{\textsf{Max-Clique}}
\newcommand{\kC}{$k$\textsf{-Clique}}

%% file: macros.tex
\usepackage{bm}
\usepackage{amsmath,amssymb,amsthm}
\usepackage[capitalise]{cleveref} 
\usepackage{xcolor}
\usepackage{etex} 
\usepackage{xspace}

\theoremstyle{plain}
\newtheorem{theorem}{Theorem}[section]
\newtheorem{lemma}[theorem]{Lemma}
\newtheorem{claim}[theorem]{Claim}
\newtheorem{proposition}[theorem]{Proposition}

\newtheorem{conjecture}[theorem]{Conjecture}

\theoremstyle{definition}

\newtheorem{problem}[theorem]{Problem}

\theoremstyle{remark}

\renewcommand{\leq}{\leqslant}
\renewcommand{\geq}{\geqslant}
\renewcommand{\subset}{\subseteq}

\newcommand{\geqsd}{\succcurlyeq}

\newcommand{\eps}{\varepsilon}
\renewcommand{\phi}{\varphi}

\renewcommand{\Pr}{\mathop{\mathbb P}}

\newcommand{\1}{\bm 1} 

\renewcommand{\R}{\mathbb R}

\newcommand{\F}{\mathbb F}
\newcommand{\N}{\mathbb N}

\newcommand{\abs}[1]{\lvert#1\rvert}
\newcommand{\Abs}[1]{\left\lvert#1\right\rvert}
\newcommand{\card}[1]{\lvert#1\rvert}
\newcommand{\norm}[1]{\lVert#1\rVert}

\newcommand{\Paren}[1]{\left(#1\right)}

\newcommand{\set}[1]{\{#1\}}

\newcommand{\inner}[1]{\langle#1\rangle}

\newcommand{\domain}[1]{[#1]}

\newcommand{\defas}{\triangleq}

\newcommand{\prob}[1]{\textnormal{\textsf{#1}}}

\newcommand{\RM}{\mathop\textnormal{RM}\nolimits}
\newcommand{\tr}{\mathop\textnormal{tr}\nolimits}
\newcommand{\Span}{\mathop\textnormal{span}\nolimits}
\newcommand{\para}{\parallel}

\newcommand{\matG}{\mathbf{G}}
\newcommand{\matL}{\mathbf{L}}
\newcommand{\lazy}{\textnormal{lazy}}

%% file: intro.tex

\section{Introduction}
Principal component analysis (PCA) is one of the most popular tools for data analytics.
PCA operates on data point vectors supported on features, and outputs orthogonal directions ({\it i.e.,} principal components) that maximize the {\it explained variance}.
A limitation of PCA is that ---in many cases of interest --- the extracted principal components (PCs) are dense. 
However, in applications such as text analysis, or gene expression analytics, 
having only a few non-zero features per extracted PC, offers significantly higher interpretabilty.
For example, in text analysis where PCs are supported on words, if they consist of only a few of them, then these words can be used to detect frequently occurring topics.

Sparse PCA  addresses the issue of interpretability directly by enforcing a sparsity constraint on the extracted PCs.
Given a matrix of centered data samples ${\bf S} \in \mathbb{R}^{n\times p}$, let us denote by $\matA =\frac{1}{n}{\bf S}{\bf S}^T$ the sample covariance matrix of the data set.
The leading sparse principal component is the solution to the following sparsity constrained, quadratic form maximization
\begin{align*}
&\max_{\|\vecx\|_2=1, \|\vecx\|_0\le k} \vecx^{\intercal}\matA\vecx  \tag{\SPCA{}}
\end{align*}
where $\|x\|_2$ is the $\ell_2$-norm and $\|\vecx\|_0$  denotes the number of nonzero entries in $\vecx$.

The objective in the above optimization is usually referred to as the {\it explained variance}.
This metric has an operational meaning: 
if a linear combination of $k$ features has high explained variance,
then it captures a representative behavior of the samples.
Typically, this means that these features ``interact" significantly with each other. 
As an example consider the case where $\matA$ is a covariance matrix of a gene expression data set.
Then, the $(i,j)$ entry of $\matA$ is a proxy for the positive or negative interaction between the $i$th and $j$th gene.
In this case, if a subset of $k$ genes ``explains" a lot of variance, then these genes have strong pairwise interactions.

There has been a large volume of work on sparse PCA: from heuristic algorithms, to statistical guarantees, and conditional approximation ratios.
Yet, there are remarkably few worst-case approximability bounds, and many questions remain open.
Does sparse PCA admit a nontrivial worst-case approximation ratio? 
Are there significant computational barriers? 
How does it relate to other problems?
In this work we take a modest first step towards a better understanding of these questions.
Our contributions are summarized below.

\begin{enumerate}
\item We show that a simple spectral technique that is popular in practice, combined with a column selection procedure, achieves an $n^{-1/3}$-approximation ratio for \SPCA{}.
\item We establish that, assuming $\P \ne \NP$, \SPCA{} does not admit a \PTAS{}.
\item We further prove that, assuming Small Set Expansion (\SSE{}) Hypothesis \cite{RS10-SSE}, \SPCA{} is hard to approximate to within {\em any} constant factor.
\item We construct an $e^{e^{\Omega(\sqrt{\log\log n})}}$ (i.e., a
  ``quasi-quasi-polynomial'') gap instance for the following commonly used SDP
  relaxation of \cite{d2007direct}
\begin{equation} 
  \begin{split}
    \max \quad & \tr(\matA \matX) \\
    \text{such that} \quad & \tr(\matX) = 1,\; \1^{\intercal} \abs \matX\1 \leq k,\; \matX \geqsd 0
  \end{split} \nonumber
\end{equation}
\end{enumerate}

\subsection{Discussion of techniques and connections to other sparsity problems}
\label{sub:discussion} 

A recurring theme in our technical discussion is the comparison of \SPCA{} to (variants of) the Densest $k$-Subgraph (\DkS) problem:
given a graph $G$, find the $k$-vertex subgraph that contains the highest number of edges.
Notice that \DkS{} can be stated as a quadratic form maximization, similar to \SPCA{}:
\begin{align*}
&\max_{\vecx \in \{0,1\}^n, \|\vecx\|_0 \le k} \vecx^T\matA\vecx  \tag{\DkS}
\end{align*}

The connection between the two problems has been observed before.
For example, it has been noted by many authors that the \kC{} problem, a decision variant of \DkS{}\footnote{Notice that \kC{} is an exact variant of both \MC{} and \DkS{}.
By now, the inapproximability of \MC{} is relatively well understood (e.g. \cite{Hastad96-clique_n^1-eps, Khot01-clique_coloring, Zuckerman07-clique_NP}),
but these results do not translate to inapproximability of \DkS{} (or \SPCA{}).}, reduces to solving \SPCA{} exactly, thus the latter is \NP-hard. Then, the Planted Clique, an average-case variant of \DkS{}, was recently used to establish statistical recovery hardness results in the sparse spiked-covariance model \cite{berthet2013optimal,berthet2013complexity, WBS14-more_SPCA_planted-hardness}.

Then, why are algorithmic and inapproximability  \DkS{} results not directly applicable to \SPCA{}?
From a computational standpoint, the most important difference between the two problems is the restriction on the input matrix $\matA$:
In \SPCA{}, $\matA$ is required to be positive semi-definite, whereas in \DkS{}, $\matA$ is required to be entry-wise nonnegative. 

With the above comparison to \DkS{} in mind, we are now ready to discuss our results and techniques.

\paragraph{$n^{-1/3}$-approximation algorithm}

Our $n^{-1/3}$-approximation scheme outputs the best solution among the following three procedures: 
{\it i)} pick the best standard basis vector; 
{\it ii)} pick the largest $k$ entries in any column vector of $\matA$; and
{\it iii)} pick the largest $k$ entries of the leading eigenvector of $\matA$.

Our algorithm is  inspired by (but is substantially different from) a combinatorial $\Omega\left(n^{-1/3}\right)$-approximation algorithm for \DkS{}, due to Feige et al. \cite{FPK01-n_to_one_third}.
The aforementioned ratio for \DkS{} was further improved in the same paper to $\Omega\left(n^{-1/3+\eps}\right)$, and later
to $\Omega\left(n^{-1/4+\eps}\right)$ \cite{BCCFV10-DkS}.
It is an open question whether similar ideas can improve the approximation guarantees for \SPCA{}.

\paragraph{\NP-hardness}
Our \NP-hardness of approximation reduction begins from MAX-E2SAT-$d$, 
the problem of maximizing the number of satisfied clauses in a CNF formula,
where every clause contains exactly two distinct variables, 
and every variable appears in exactly $d=O(1)$ clauses.
We set $A_{i,j}$ to be higher if literals $i$ and $j$ satisfy some clause,
and a consistent assignment is ensured by having large negative values whenever indices $i$ and $j$ 
correspond to a literal and its negation.
A PSD matrix is obtained by adding a large multiple of the identity.
As we discuss below, this last step seems to be the main obstacle to obtaining a stronger inapproximability factor.

Interestingly, this result highlights an important difference between \SPCA{} and \DkS{}:
for the latter, proving \NP-hardness to within any constant factor remains a major open problem.

\paragraph{The PSD challenge}
The biggest challenge to obtaining inapproximability results for \SPCA{}, from say \DkS{}, is achieving $\matA \succcurlyeq 0$.
One naive approach is to add a large multiple of the identity matrix and force diagonal dominance (as we do for our \NP-hardness result). 
Unfortunately, this ruins our inapproximability factor: the large entries on the diagonal outweigh the interactions between different features.
In particular, {\em every} vector achieves a reasonably high explained variance. 

A second approach to obtain a PSD matrix is by squaring the adjacency matrix.
When we start from Planted Clique, or other known hard \DkS{} instances (e.g. \cite{BCVGZ12-SDP-DkS, AAMMW11, Kho06, BKRW15-DkS}),
squaring the adjacency matrix  gives weak inapproximability results, as in the case of \cite{berthet2013optimal,berthet2013complexity, KNV15-SDP-Sparse_PCA, WBS14-more_SPCA_planted-hardness} 
(see also discussion of impossibility results for the sparse spiked covariance model below).
To understand why, it is helpful to consider random walks on regular graphs.
The density of a subgraph is proportional to the probability that a length-$1$ random walk remains in the subgraph. (Thus the densest $k$-subgraph is also the least {\em expanding} $k$-subgraph.) 
Similarly, when we restrict $\matA^2$ to the same $k$-tuple of vertices, the density corresponds to the probability of remaining in the subgraph after a random walk of length $2$.
Intuitively, squaring the adjacency matrices of the instances mentioned at the beginning of this paragraph is ineffective, 
because even their dense subgraphs are very expanding:
most length-$2$ walks that start and end inside the densest $k$-subgraph, take their middle step outside the subgraph. 
Thus the density of the subgraph has only a small effect on the density with respect to $\matA^2$.
To overcome this difficulty, we want the ``good'' subgraph to have very small expansion.

\paragraph{\SSE-hardness and SDP gap}
The Small Set Expansion Hypothesis \cite{RS10-SSE} postulates that it is hard to find a linear-size subgraph with a very small expansion.
Intuitively, if the expansion of a particular $k$-subgraph is sufficiently small, then, even after taking two steps, the random walk should remain inside the subgraph; thus the corresponding $k$-sparse vector should continue to give an exceptionally high value for \DkS/\SPCA{} with $\matA^2$. 
To formalize this intuition, we apply a recent result of Raghavendra and Schramm \cite{raghavendra2013gap} on the expansion of random walk graphs.

Finally, the gap for the standard semidefinite program for \SPCA{} builds on known integrality gap instances for \SSE, in particular the Short Code graph \cite{BarakGHMRS12}.
Notice that the ``quasi-quasi-polynomial'' factor ($e^{e^{\Omega(\sqrt{\log\log n})}}$) is slightly smaller than polynomial and ``quasi-polynomial'' ($e^{\Omega(\sqrt{\log n})}$) factors, but much larger than polylogarithmic.

\paragraph{Additive \PTAS{}}
To complete the picture of our current understanding of worst-case approximability of \SPCA{},
let us also mention a recent additive \PTAS{} due to \cite{asteris2015sparse}.
By additive \PTAS{}, we mean that if all the entries of $\matA$ are bounded in [-1, 1],
the optimum explained variance can be approximated in polynomial time to within an additive error of $\eps k$, for any constant $\eps > 0$
(compare to an optimum of at most $k$ in the case of an all-ones $k \times k$ submatrix).
In contrast, note that a corresponding additive \PTAS{} for \DkS{} is unlikely \cite{BKRW15-DkS}.

\ignore{

We leave many interesting open questions. 
Our $n^{1/3}$-approximation algorithm is extremely simple; can we obtain better approximation ratios with the use of more sophisticated, spectral, or convex relaxation tools?
Is it possible to approximate \SPCA{} within a poly-logarithmic factor?
Does \SPCA{} inherit any further approximation or hardness results from Small Set Expansion or Densest $k$-Subgraph?

}

\subsection{Related work} 

\paragraph{Heuristics and algorithms}
The algorithmic tapestry for sparse PCA is rich and diverse.
Early heuristics used rotation and thresholding of eigenvectors \cite{kaiser1958varimax,jolliffe1995rotation, cadima1995loading} and  LASSO heuristics \cite{Ando2009420,jolliffe2003modified}.
Then, in \cite{zou2006sparse}, a nonconvex $\ell_1$ penalized approximation, re-generated a lot of interest in the problem.
A great variety of greedy, spectral, and nonconvex heuristics were presented in the past decade \cite{sriperumbudur2007sparse,
moghaddam2006spectral,
moghaddam2007fast,
shen2008sparse,
journee2010generalized,
yuan2013truncated,
kuleshov2013fast}.
There has also been a significant body of work on semidefinite programming (SDP) approaches \cite{d2007direct,zhang2012sparse,d2008optimal, d2012approximation}. 
Some recent works established conditional approximation guarantees for sparse PCA using spectral $\epsilon$-net search algorithms, under the assumption of a decaying matrix spectrum \cite{asteris2014nonnegative,asteris2015sparse}.

\paragraph{Sparse spiked covariance model and recent impossibility results}
The performance of many algorithms has been analyzed under the {\em sparse spiked covariance} and related models. 
For example, under this model Amini et al. \cite{amini2008high}  develop the first theoretical guarantees for simple thresholding and the SDP of \cite{d2007direct}.
Several statistical analyses were carried for more general settings, while using a variety of different algorithms, under various metrics of interest \cite{ma2013sparse,d2012approximation,cai2012sparse,cai2013optimal,deshpande2013sparse}.

In this model, we are collecting samples from a distribution with a covariance matrix that is equal to the identity plus a sparse rank-1 matrix (the spike).
Our goal is to identify (or detect) the rank-1 sparse ``spike'' from the samples.
If we could observe the true covariance matrix the algorithmic task would be trivial.
However, when the input to this problem is a finite number of samples, then there exist sharp information theoretic, and computational barriers on the identifiability of the spike.

A recent celebrated line of works \cite{berthet2013optimal,berthet2013complexity, WBS14-more_SPCA_planted-hardness}, initiated by Berthet and Rigollet, establish a gap between the threshold of samples where detection is information theoretically possible, and that were it is computationally feasible, assuming hardness of the Planted Clique problem;
a similar result was also obtained by Krauthgamer et al.~\cite{KNV15-SDP-Sparse_PCA} with respect to the standard SDP.
However,  these results do not show a significant gap between the optimal value of the primal objective ($\vecx^T\matA\vecx$) and what is achievable by efficient algorithms.
In particular, {\em none of these results rule out a polynomial time algorithm (in particular, an algorithm does not even attempt to detect the spike) that achieves a $(1-o(1))$-approximation of the optimal explained variance.} 
This weak inapproximability seems to be a fundamental barrier of reductions from Planted Clique (see also discussion in Section \ref{sub:discussion}).

We should also mention some recent inapproximability results in the general case where $\matA$ is not necessarily positive semi-definite (PSD) \cite{magdon:spca:hardness}.
(Recall that in typical applications $\matA$ is a covariance matrix and thus necessarily PSD.)
We note that in this general matrix setting, it is even hopeless to determine, in polynomial time, the sign of the optimal value, unless $\P = \NP$.

\ifFULL
\subsection{Organization}
Our approximation algorithm is described in \cref{sec:alg}.
In \cref{sec:np-hardness} we prove our \NP-hardness result, and our \SSE{}-hardness result appears in \cref{sec:sse-hardness}.
Finally, in \cref{sec:gap} we prove the quasi-quasi-polynomial gap for the standard SDP. 
For completeness, we also briefly describe in \cref{sec:aPTAS} the additive \PTAS{} due to \cite{asteris2015sparse},
and shortly discuss in \cref{sec:non-PSD} the case where the input matrix is not PSD.
\fi

%% file: n_to_one-third.tex

\section{$n^{-1/3}$-approximation algorithm} \label{sec:alg}
\begin{theorem}
\SPCA{} can be approximated to within $n^{-1/3}$
in deterministic polynomial time.\label{thm:n^=00007B1/3=00007D}
\end{theorem}
The rest of this section is devoted to the proof of \cref{thm:n^=00007B1/3=00007D}.
Our algorithm takes the best of two options: a truncation of one of
$\matA$'s columns in the standard basis, and a truncation of one of $\matA$'s
eigenvectors. We present and analyze the guarantees for each algorithm,
and then show that together they give the bound on the approximation
ratio.

Let $\mathbf{y}_{*}$ denote an optimum solution to the {\sc Sparse
PCA} instance, and let $OPT=\mathbf{y}_{*}^{\intercal}\matA\mathbf{y}_{*}$
denote its value.

\subsection{Truncation in the standard basis}

\paragraph{Algorithm 1}

For each $i\in\left[n\right]$, let $\matA_{\cdot,i}$ be the $i$-th
column of $\matA$, and let $\mathbf{x}_{i}$ be the unit-norm, $k$-sparse
truncation of $\matA_{\cdot,i}$. That is, let 
\[
\left[\hat{\mathbf{x}}_{i}\right]_{j}=\begin{cases}
\matA_{i,j} & \mbox{if \ensuremath{\left|\matA_{i,j}\right|}\,is one of the \ensuremath{k} largest (in absolute value) entries of \ensuremath{\matA_{\cdot,i}}}\\
0 & \mbox{otherwise}
\end{cases}
\]
and $\mathbf{x}_{i}=\hat{\mathbf{x}}_{i}/\left\Vert \hat{\mathbf{x}}_{i}\right\Vert _{2}$. 

Return the best out of all $\mathbf{x}_{i}$'s and $\mathbf{e}_{i}$'s,
where $\mathbf{e}_{i}$ is the $i$-th standard basis vector.\\

\begin{lemma}
Algorithm 1 returns a solution with value $V_{1}\left(\matA,k\right)\geq\frac{OPT}{\sqrt{k}}$\end{lemma}
\begin{proof}
First, we claim that for each $i$, $\mathbf{x}_{i}$ maximizes $\mathbf{e}_{i}^{\intercal}\matA\mathbf{x}_{i}$
among all feasible ($k$-sparse and unit-norm) vectors. By Cauchy-Schwartz
inequality, for any choice of support $S$ of size $k$, the unit-norm
vector that maximizes the inner product with $\matA_{\cdot,i}$ is the
restriction of $\matA_{\cdot,i}$ to $S$, normalized. The inner product
is thus $\sqrt{\sum_{j\in S}\matA_{j,i}^{2}}$; this is indeed maximized
when $S$ is the set of entries with largest absolute value.

Now, rewrite $\mathbf{y}_{*}=\sum y_{i}\mathbf{e}_{i}$ as a linear
combination of (at most $k$) standard basis vectors. By Cauchy-Schwartz
inequality, we have
\begin{gather*}
OPT = \sum y_{i}\left(\mathbf{e}_{i}^{\intercal}\matA\mathbf{y}_{*}\right)\leq\sqrt{\sum y_{i}^{2}}\sqrt{\sum\left(\mathbf{e}_{i}^{\intercal}\matA\mathbf{y}_{*}\right)^{2}}.
\end{gather*}
Plugging in $\sqrt{\sum y_{i}^{2}}=\left\Vert \mathbf{y}_{*}\right\Vert _{2}=1$,
we get
\[
OPT\leq\sqrt{\sum\left(\mathbf{e}_{i}^{\intercal}\matA\mathbf{y}_{*}\right)^{2}}\leq\sqrt{k}\max_{i}\mathbf{e}_{i}^{\intercal}\matA\mathbf{y}_{*}.
\]
In particular, this means that for some $i$, then $\mathbf{e}_{i}^{\intercal}\matA\mathbf{x}_{i}\geq OPT/\sqrt{k}$, where $\vecx_i$, is as defined above.

Finally, since $\matA\succcurlyeq0$, we have
\[
0\leq\left(\mathbf{e}_{i}-\mathbf{x}_{i}\right)^{\intercal}\matA\left(\mathbf{e}_{i}-\mathbf{x}_{i}\right)=\mathbf{e}_{i}^{\intercal}\matA\mathbf{e}_{i}+\mathbf{x}_{i}^{\intercal}\matA\mathbf{x}_{i}-2\mathbf{e}_{i}^{\intercal}\matA\mathbf{x}_{i}.
\]
Rearranging, we get
\[
\max\left\{ \mathbf{e}_{i}^{\intercal}\matA\mathbf{e}_{i},\mathbf{x}_{i}^{\intercal}\matA\mathbf{x}_{i}\right\} \geq OPT/\sqrt{k}.
\]
 
\end{proof}

\subsection{Truncation in the eigenspace basis}

\paragraph{Algorithm 2}

Let $\left(\mathbf{v}_{1},\lambda_{1}\right)$ be the top eigenvector
and eigenvalue of $\matA$.
Return the unit-norm, $k$-sparse truncation of $\mathbf{v}_{1}$.
That is, let
\[
\left[\hat{\mathbf{x}}\right]_{j}=\begin{cases}
\left[\mathbf{v}_{1}\right]_{j} & \mbox{if \ensuremath{\left[\mathbf{v}_{1}\right]_{j}}\,is one of the \ensuremath{k} largest (in absolute value) entries of \ensuremath{\left[\mathbf{v}_{1}\right]_{j}}}\\
0 & \mbox{otherwise}
\end{cases}
\]
and $\mathbf{x}=\hat{\mathbf{x}}/\left\Vert \hat{\mathbf{x}}\right\Vert _{2}$.
Return $\mathbf{x}$. 
\begin{lemma}
Algorithm 2 returns a solution with value $V_{2}\left(\matA,k\right)\geq\frac{k}{n}OPT$.\end{lemma}
\begin{proof}
First, notice that 
\begin{gather*}
\mathbf{x}^{\intercal}\matA\mathbf{v}_{1}=\lambda_{1}\mathbf{x}^{\intercal}\mathbf{v}_{1}=\lambda_{1}\mathbf{x}^{\top}\hat{\mathbf{x}}=\lambda_{1}\cdot\left\Vert \hat{\mathbf{x}}\right\Vert _{2}\geq\lambda_{1}\sqrt{k/n},
\end{gather*}
where the last inequality follows by the greedy construction of $\hat{\mathbf{x}}$.
Since $\matA\succcurlyeq0$, it induces an inner product over $\mathbb{R}^n$. Thus we can 
apply the Cauchy Schwartz inequality to get:
\begin{gather*}
\lambda_{1}\sqrt{k/n}\leq\mathbf{x}^{\top}\matA\mathbf{v}_{1}\leq\sqrt{\mathbf{x}^{\top}\matA\mathbf{x}}\cdot\sqrt{\mathbf{v}_{1}^{\top}\matA\mathbf{v}_{1}}=\sqrt{\mathbf{x}^{\top}\matA\mathbf{x}}\cdot\sqrt{\lambda_{1}}.
\end{gather*}
Rearranging, we have
\[
\mathbf{x}^{\top}\matA\mathbf{x}\geq\frac{k}{n}\lambda_{1}.
\]

Finally, to complete the proof recall that $\lambda_{1}=\mathbf{v}_{1}^{\top}\matA\mathbf{v}_{1}\geq OPT$
since $\mathbf{v}_{1}$ maximizes the objective function among all
(not necessarily $k$-sparse) unit-norm vectors.
\end{proof}

\subsection{Putting it altogether}

Our final algorithm simply takes the best out of the outputs of Algorithms
1 and 2. We now have
\begin{eqnarray*}
V\left(\matA,k\right) & = & \max\left\{ V_{1}\left(\matA,k\right),V_{2}\left(\matA,k\right)\right\} \\
 & \geq & \left(V_{1}\left(\matA,k\right)\right)^{2/3}\cdot\left(V_{2}\left(\matA,k\right)\right)^{1/3}\\
 & \geq & \frac{OPT^{2/3}}{k^{1/3}}\cdot\frac{k^{1/3}}{n^{1/3}}OPT^{1/3}=OPT/n^{1/3}
\end{eqnarray*}

%% file: NP-hardness.tex

\section{\NP-hardness}\label{sec:np-hardness}
\begin{theorem}
There exists a constant $\eps>0$ such that \SPCA{}
is \NP-hard to approximate to within $\left(1-\eps\right)$.\end{theorem}
\begin{proof}

We reduce from MAX-E2SAT-$d$: given a 2CNF over $n$ variables where
every variable appears in exactly $d$ distinct clauses, maximize
the number of satisfied clauses. 

\begin{lemma}\label{lem:2-sat}
There exist constants
$0<s<c<1$  and $d$ such that given a MAX-E2SAT-$d$ instance over $n$ clauses,
it is \NP-hard to decide whether 
at least $cn$ clauses can be satisfied (``yes'' case), 
or at most $sn$ (``no'' case).
\end{lemma}

\cref{lem:2-sat} follows from standard techniques. 
\ifFULL
We briefly sketch the proof below for completeness.

\begin{proof}[Proof sketch of \cref{lem:2-sat}]
By, e.g. \cite{feige98}, MAX-3SAT-$5$ is $NP$-hard to approximate to within some constant factor.
We can convert each 3SAT clause $C = (x \vee y \vee z)$ into $10$ 2SAT clauses (introducing one additional variable $h_C$),
\[ (x) \wedge (y) \wedge (z) \wedge (h_C) \wedge \\ 
(\neg x \vee \neg y) \wedge (\neg x \vee \neg z) \wedge (\neg y \vee \neg z) \\
(x \vee \neg h_C) \wedge (y \vee \neg h_C)  \wedge (z \vee \neg h_C)  \]
with the following guarantee: the optimal assignment to the 2SAT instance satisfies at most $7$ out of $10$ clauses for every satisfied 3SAT clause, and at most $6$ out of $10$ clauses for every unsatisfied 3SAT clause (e.g. \cite{muli_notes}).

This establishes the result for MAX-2SAT with bounded degree. Add a linear number of variables and trivially satisfied clauses to get a MAX-E2SAT-$d$ instance.
\end{proof}
\fi

Given a 2CNF $\psi$, we construct a symmetric $2n\times2n$ matrix
$\matA^{(0)}=\matA^{(0)}\left(\psi\right)$ as follows: every row/column corresponds
to a literal of $\psi$; if row $i$ and column $j$ correspond to
an assignment that satisfies some clause, then $\matA_{i,j}^{(0)}=1$, and
$\matA_{i,j}^{(0)}=0$ otherwise. Let ${\cal Y}$ denote the set of vectors
that correspond to legal assignments to $\psi$, i.e. 
\[
{\cal Y}=\left\{ \mathbf{y}\colon\substack{\left|\mathbf{y}\right|_{2}=1;\,\mathbf{y}\in\left\{ 0,1/\sqrt{n}\right\} ^{2n};\\
\mbox{\ensuremath{\forall}i\,\,\ \ensuremath{\mathbf{y}_{x_{i}}=0\iff\mathbf{y}_{\neg x_{i}}=1/\sqrt{n}}}
}
\right\} 
\]
By \cref{lem:2-sat} it is \NP-hard to distinguish between 
\[
\mbox{``yes'':}\max_{\mathbf{y}\in{\cal Y}}\mathbf{y}^{\intercal}\matA^{(0)}\mathbf{y}\geq c~~~~~~~~~~~~~~~~~\mbox{``no'':}\max_{\mathbf{y}\in{\cal Y}}\mathbf{y}^{\intercal}\matA^{(0)}\mathbf{y}\leq s.
\]

The proof continues by adding the following matrices to $\matA^{(0)}$:
a matrix $\matC$ with large negative entries that enforces a consistent
assignment; a larger scalar times the identity matrix that ensures
our input is PSD; and an even larger (yet still constant) scalar times
the all-ones matrix that guarantees the optimal solution uses a large
support. While adding these matrices preserves the qualitative
properties of the instance, they significantly weaken our inapproximability
factor.

\subsubsection*{Enforcing a consistent assignment}

Our first step is to enforce consistency using the objective function
instead of restricting the input to be from $\cal{Y}$. Let $\matC_{i,j}=-2d$ if $i$
and $j$ correspond to a literal and its negation, and $\matC_{i,j}=0$,
otherwise. We claim that among all unit-norm vectors $\mathbf{z}\in\left\{ 0,1/\sqrt{n}\right\} ^{2n}$,
the objective $\mathbf{z}\left(\matA^{(0)}+\matC\right)\mathbf{z}$ is maximized
by some legal assignment $\mathbf{z}^{*}\in{\cal Y}$. Assume by contradiction
that the objective is maximized by some $\mathbf{z}$ which assigns
$1/\sqrt{n}$ to some variable $x_{i}$ and its negation; since $\mathbf{z}$
is exactly $n$-sparse, it must also assign $0$ to another variable
$x_{j}$ and its negation. However, the objective value can be increased
by considering $\mathbf{z}'$ which assigns $1/\sqrt{n}$ to $x_{i}$
and $x_{j}$, $0$ to their negations, and is equal to $\mathbf{z}$
everywhere else. Therefore, for $\matA^{(1)}\defas \matA^{(0)}+\matC$, we have 
\[
\mbox{``yes'':}\max_{\substack{\left|\mathbf{z}\right|_{2}=1\\
\mathbf{z}\in\left\{ 0,1/\sqrt{n}\right\} ^{2n}
}
}\mathbf{z}^{\intercal}\matA^{(1)}\mathbf{z}\geq c~~~~~~~~~~~~~~~~~\mbox{``no'':}\max_{\substack{\left|\mathbf{z}\right|_{2}=1\\
\mathbf{z}\in\left\{ 0,1/\sqrt{n}\right\} ^{2n}
}
}\mathbf{z}^{\intercal}\matA^{(1)}\mathbf{z}\leq s.
\]

\subsubsection*{PSD input}

$\matA^{(1)}$ is not a legitimate input to \SPCA{} because it
is not be positive semi-definite. Fortunately, $\matA^{(2)}\defas 3d\matI+\matA^{(1)}$
is positive semi-definite because it is symmetric and diagonally-dominant.
The identity matrix adds exactly $1$ to the objective function for
any input. Therefore we also have
\begin{align}
\mbox{``yes'':} & \max_{\substack{\left|\mathbf{z}\right|_{2}=1\\
\mathbf{z}\in\left\{ 0,1/\sqrt{n}\right\} ^{2n}
}
}\mathbf{z}^{\intercal}\matA^{(2)}\mathbf{z}  \geq  3d+c
~~~~~~~~~~~~~~~~~
\mbox{``no'':} & \max_{\substack{\left|\mathbf{z}\right|_{2}=1\\
\mathbf{z}\in\left\{ 0,1/\sqrt{n}\right\} ^{2n}
}
}\mathbf{z}^{\intercal}\matA^{(2)}\mathbf{z}  \leq  3d+s\label{eq:withPSD}
\end{align}

\subsubsection*{Enforcing a (nearly) $n$-uniform optimum}

Now, we would of course like to replace $\left\{ 0,1/\sqrt{n}\right\} ^{2n}$
with the set of all $n$-sparse vectors, while maintaining (approximately)
the same optima. Consider the positive semi-definite matrix $\matJ=\mathbf{1}\mathbf{1}^{\intercal}$;
the objective $\mathbf{x}^{\intercal}\matJ\mathbf{x}=\left|\mathbf{x}\right|_{1}^{2}$
is maximized by an $n$-uniform vector in $\left\{ 0,1/\sqrt{n}\right\} ^{2n}$.

We define our final hard instance input matrix to be $\matA^{(3)}\defas \frac{\alpha}{n}\matJ+\matA^{(2)}$,
for a sufficiently large (but constant) $\alpha$.
\ifFULL
As we show below, the objective is now maximized by
a vector $\mathbf{x}$ that is approximately $n$-uniform.

Formally, observe that $\matA^{(2)}$ induces an inner product over $\mathbb{R}^{n}$;
thus for any $\left|\mathbf{x}\right|_{2}^{2}=\left|\mathbf{z}\right|_{2}^{2}=1$
we can use the Cauchy-Schwartz inequality to get: 
\begin{align*}
\mathbf{x}^{\intercal}\matA^{(2)}\mathbf{x}-\mathbf{z}^{\intercal}\matA^{(2)}\mathbf{z} & =\left(\mathbf{x}-\mathbf{z}\right)^{\intercal}\matA^{(2)}\left(\mathbf{x}+\mathbf{z}\right)\\
 & \leq\sqrt{\left(\mathbf{x}-\mathbf{z}\right)^{\intercal}\matA^{(2)}\left(\mathbf{x}-\mathbf{z}\right)}\cdot\sqrt{\left(\mathbf{x}+\mathbf{z}\right)^{\intercal}\matA^{(2)}\left(\mathbf{x}+\mathbf{z}\right)}\\
 & \leq\left\Vert \matA^{(2)}\right\Vert _{2}^{2}\cdot\left|\mathbf{x}+\mathbf{z}\right|_{2}\left|\mathbf{x}-\mathbf{z}\right|_{2},
\end{align*}
where $\left\Vert \matA^{(2)}\right\Vert _{2}$ is the $l^{2}$ operator
norm of $\matA^{(2)}$, and is bounded by: 
\begin{gather*}
\left\Vert \matA^{(2)}\right\Vert _{2}^{2}=\max_{\left|\mathbf{x}\right|_{2}^{2}=1}\mathbf{x}^{\intercal}\matA^{(2)}\mathbf{x}=3d+\max_{\left|\mathbf{x}\right|_{2}^{2}=1}\mathbf{x}^{\intercal}\matA^{(1)}\mathbf{x}<5d+\max_{\left|\mathbf{x}\right|_{2}^{2}=1}\mathbf{x}^{\intercal}\matA^{(0)}\mathbf{x}\leq6d.
\end{gather*}
By triangle inequality, $\left|\mathbf{x}+\mathbf{z}\right|_{2}\leq2$,
and therefore 
\begin{equation}
\mathbf{x}^{\intercal}\matA^{(2)}\mathbf{x}-\mathbf{z}^{\intercal}\matA^{(2)}\mathbf{z}\leq12d\left|\mathbf{x}-\mathbf{z}\right|_{2}.\label{eq:xAx-zAz-1}
\end{equation}

Suppose further that $\mathbf{z}$ is a rounding of $\mathbf{x}$ to
$\left\{ 0,1/\sqrt{n}\right\} ^{2n}$. In particular, $\supp\left(\mathbf{x}\right)\subseteq\supp\left(\mathbf{z}\right)$
(we have equality if $\mathbf{x}$ is exactly $n$-sparse) and $\mathbf{x}^{\intercal}\mathbf{z}=\lambda_{\mathbf{z}}\geq0$.
Let us decompose $\mathbf{x}=\lambda_{\mathbf{z}}\mathbf{z}+\lambda_{\mathbf{w}}\mathbf{w}$
where $\mathbf{w}$ is a unit-norm vector orthogonal to $\mathbf{z}$
(i.e. $\left|\mathbf{w}\right|_{2}=1$ and $\mathbf{w}^{\intercal}\mathbf{z}=0$).
Since all the vectors have unit norm, $\lambda_{\mathbf{z}}^{2}+\lambda_{\mathbf{w}}^{2}=\lambda_{\mathbf{z}}^{2}\left|\mathbf{z}\right|_{2}^{2}+\lambda_{\mathbf{w}}^{2}\left|\mathbf{w}\right|_{2}^{2}=\left|\mathbf{x}\right|_{2}^{2}=1$.
We can now write the difference between $\mathbf{x}$ and $\mathbf{z}$
as,
\begin{align}
\left|\mathbf{x}-\mathbf{z}\right|_{2}^{2} & = \left|\left(1-\lambda_{\mathbf{z}}\right)\mathbf{z}+\sqrt{1-\lambda_{\mathbf{z}}^{2}}\cdot\mathbf{w}\right|_{2}^{2}\nonumber \\
 & = \left(1-\lambda_{\mathbf{z}}\right)^{2}+\left(1-\lambda_{\mathbf{z}}^{2}\right)\nonumber \\
 & \leq 2\left(1-\lambda_{\mathbf{z}}^{2}\right).\label{eq:|x-z|2-1}
\end{align}
Since $\supp\left(\mathbf{x}\right)\subseteq\supp\left(\mathbf{z}\right)$,
we also have that $\supp\left(\mathbf{w}\right)\subseteq\supp\left(\mathbf{z}\right)$.
Thus $\mathbf{w}^{\intercal}\mathbf{z}=0$ is equivalent to $\mathbf{w}^{\intercal}\mathbf{1}=0$.
We therefore have: 
\begin{align}
\mathbf{z}^{\intercal}\matJ\mathbf{z}-\mathbf{x}^{\intercal}J\mathbf{x} & =\mathbf{z}^{\intercal}\matJ\mathbf{z}-\left(\lambda_{\mathbf{z}}\mathbf{z}+\lambda_{\mathbf{w}}\mathbf{w}\right)^{\intercal}\matJ\left(\lambda_{\mathbf{z}}\mathbf{z}+\lambda_{\mathbf{w}}\mathbf{w}\right)\nonumber \\
 & \underbrace{=}_{\mathbf{w}^{\intercal}\mathbf{1}=0}\mathbf{z}^{\intercal}\matJ\mathbf{z}-\left(\lambda_{\mathbf{z}}\mathbf{z}\right)^{\intercal}\matJ\left(\lambda_{\mathbf{z}}\mathbf{z}\right)\nonumber \\
 & =\left(1-\lambda_{\mathbf{z}}^{2}\right)\cdot\mathbf{z}^{\intercal}\matJ\mathbf{z}\nonumber \\
 & \underbrace{\geq}_{\cref{eq:|x-z|2-1}}\frac{\left|\mathbf{x}-\mathbf{z}\right|_{2}^{2}}{2}\cdot\left\Vert \matJ\right\Vert _{2}=\frac{n}{2}\left|\mathbf{x}-\mathbf{z}\right|_{2}^{2}.\label{eq:xJx-zJz-1}
\end{align}
Recall that $\matA^{(3)}=\frac{\alpha}{n}\matJ+\matA^{(2)}$. Combining \cref{eq:xAx-zAz-1}
and \cref{eq:xJx-zJz-1}, we have that for every $n$-sparse, unit-norm
$\mathbf{x}$, 
\begin{gather}
\max_{\substack{\left|\mathbf{z}\right|_{2}^{2}=1\\
\mathbf{z}\in\left\{ 0,1/\sqrt{n}\right\} ^{2n}
}
}\mathbf{z}^{\intercal}\matA^{(3)}\mathbf{z}.\label{eq:x(J+dI+AG)x-1} - 
\mathbf{x}^{\intercal}\matA^{(3)}\mathbf{x} \geq \frac{\alpha}{2}\cdot\left|\mathbf{x}-\mathbf{z}\right|_{2}^{2} -12d\cdot\left|\mathbf{x}-\mathbf{z}\right|_{2}
\end{gather}
Let $\alpha \defas 144d^{2}/\left(c-s\right)$. Then, 
\begin{align*}
{\color{red} \frac{\alpha}{2}\cdot\left|\mathbf{x}-\mathbf{z}\right|_{2}^{2}}
	- {\color{magenta} 12d\cdot\left|\mathbf{x}-\mathbf{z}\right|_{2}}
& =  {\color{red} \left(\frac{72}{c-s}\right)d^{2}\cdot\left|\mathbf{x}-\mathbf{z}\right|_{2}^{2}}
	-{\color{magenta}12d\cdot\left|\mathbf{x}-\mathbf{z}\right|_{2}}\\
 & =  \left(\frac{2}{c-s}\right) 
 	\left({\color{red}36d^{2}\cdot\left|\mathbf{x}-\mathbf{z}\right|_{2}^{2}}
 	- {\color{magenta} 6d\cdot\left|\mathbf{x}-\mathbf{z}\right|_{2}\left(c-s\right)}
	 +{\color{blue} \left(\frac{c-s}{2}\right)^{2}}\right)
	 -{\color{blue} \frac{c-s}{2}}\\
 & =  \left(\frac{2}{c-s}\right)  
 	\left({\color{red} 6d\cdot\left|\mathbf{x}-\mathbf{z}\right|_{2}}-{\color{blue}\frac{c-s}{2}}\right)^{2}-{\color{blue}\frac{c-s}{2}}\\
 & \geq  {\color{blue} \frac{c-s}{2}}.
\end{align*}
Plugging into \cref{eq:x(J+dI+AG)x-1}, we have
\[
\max_{\substack{\left|\mathbf{z}\right|_{2}^{2}=1\\
\mathbf{z}\in\left\{ 0,1/\sqrt{n}\right\} ^{2n}
}
}\mathbf{z}^{\intercal}\matA^{(3)}\mathbf{z}\leq\max_{\substack{\left|\mathbf{x}\right|_{2}^{2}=1\\
\left|\mathbf{x}\right|_{0}\leq n
}
}\mathbf{x}^{\intercal}\matA^{(3)}\mathbf{x}\leq\frac{c-s}{2}+\max_{\substack{\left|\mathbf{z}\right|_{2}^{2}=1\\
\mathbf{z}\in\left\{ 0,1/\sqrt{n}\right\} ^{2n}
}
}\mathbf{z}^{\intercal}\matA^{(3)}\mathbf{z}.
\]
Finally, by \cref{eq:withPSD}, it is \NP-hard to distinguish between:
\fi

\begin{align*}
\mbox{``yes'':}\max_{\substack{\left|\mathbf{x}\right|_{2}^{2}=1\\
\left|\mathbf{x}\right|_{0}\leq n
}
}\mathbf{x}^{\intercal}\matA^{(3)}\mathbf{x} & \geq & \alpha+3d+c
~~~~~~~~~~~~~~~~~
\mbox{``no'':}\max_{\substack{\left|\mathbf{x}\right|_{2}^{2}=1\\
\left|\mathbf{x}\right|_{0}\leq n
}
}\mathbf{x}^{\intercal}\matA^{(3)}\mathbf{x} & \leq & \alpha+3d+\frac{c+s}{2}.
\end{align*}

\end{proof}

%% file: sse-hardness.tex

\section{Small-Set Expansion hardness} \label{sec:sse-hardness}

Throughout this section, we will consider edge-weighted $1$-regular graphs $G =
(V, E)$, whose adjacency matrix/probability transition matrix $\matG$ has every
row sum equal to $1$.

Recall that for a $1$-regular graph $G = (V, E)$ on $n$ vertices, the expansion of
$S\subset V$ is
\[ \Phi_G(S) \defas \frac{\card{E(S,V\setminus S)}}{\card S}, \]
where $\card{E(S,T)} \defas \sum_{i\in S, j\in T} \matG_{ij}$ denotes the total
weight of edges with one end point in $S$ and one end point in $T$.
The expansion profile  of $G$ is
\[ \Phi_G(\delta) \defas \min_{S: \card S\leq \delta n} \Phi_G(S). \]

Recall the Small-Set Expansion Hypothesis \cite{RaghavendraS10}%
\ifFULL
\footnote{This
  formulation comes from the full version of the paper on Prasad Raghavendra's
  homepage.  This formulation has a different soundness condition than the one
  in the conference version of \cite{RaghavendraS10}.  Furthermore,
\cite{RaghavendraST12} shows that the two formulations are equivalent.}:

\begin{problem}[$\prob{SSE}(\eta,\delta)$]
  \label{prob:sse}
Given a regular graph $G = (V,E)$, distinguish between the following two cases:
\begin{enumerate}
  \item Yes: Some subset $S\subset V$ with $\card S = \delta n$ has $\Phi_G(S)
    \leq \eta$
  \item No: Any set $S\subset V$ with $\card S \leq 2\delta n$ has $\Phi_G(S)
    \geq 1-\eta$
\end{enumerate}
\end{problem}

\begin{conjecture}[Small-Set Expansion Hypothesis \cite{RaghavendraS10}]
  \label{conj:sseh}
  For any $\eta > 0$, there is $\delta > 0$ such that $\prob{SSE}(\eta,
  \delta)$ is \NP-hard.
\end{conjecture}

There is little consensus among researchers whether this conjecture is true.
At any rate, if the conjecture turns out to be false, significantly new
algorithmic or analytic ideas will be needed.
See e.g.~\cite{AroraBS10,BarakBHKSZ12} on efforts to refute the conjecture and
pointers to the literature.

It is more convenient to work with the following version of Small-Set
Expansion, where in the No case the subset size can be an arbitrarily
large constant multiple of the subset size in the Yes case.

\begin{problem}[$\prob{SSE}(\eta,\delta,M)$] \label{prob:sse-m}
Given a regular graph $G = (V,E)$, distinguish between the following two cases:
\begin{enumerate}
  \item Yes: Some subset $S\subset V$ with $\card S = \delta n$ has $\Phi_G(S)
    \leq \eta$
  \item No: Any set $S\subset V$ with $\card S \leq M\delta n$ has $\Phi_G(S)
    \geq 1-\eta$
\end{enumerate}
\end{problem}

The following reduction in \cite[Proposition~5.8]{RaghavendraST12} shows that
the two versions of Small-Set Expansion are equivalent.

\begin{claim} \label{claim:sse-m}
  For all $\eta, \delta > 0$, $M \geq 1$, there is a polynomial time reduction
  from $\prob{SSE}(\eta/M, \delta)$ to $\prob{SSE}(\eta, \delta, M)$.
\end{claim}

We note that our statement is slightly different from
\cite[Proposition~5.8]{RaghavendraST12}, due to our different version of
Small-Set Expansion Hypothesis, but the proof of the above claim is the same.
\fi

We will use the following lemma from \cite{RaghavendraS14}. 
Here the lazy random walk $\matG_\lazy$ corresponds to staying at the current
vertex with probability $1/2$, otherwise moving according to the probability
transition matrix $\matG$.
Therefore the probability transition matrix is given by $\matG_\lazy \defas
(\matI + \matG)/2$.
For any $t \in \N$, define the $t$-step lazy random walk as $\matG_\lazy^t
\defas (\matG_\lazy)^t$, and let $G_\lazy^t$ denote the corresponding graph.

\begin{lemma}{\cite[Lemma~13]{RaghavendraS14}} \label{lemma:power-expand}
For all $t\in \N$ and $\eta,\delta\in (0,1]$,
\[ \Phi_{G_\lazy^t}(\delta) \geq
\min\Paren{1-\Paren{1-\frac{\Phi_G^2(4\delta/\eta)}{32} }^t, 1-\eta} . \]
\end{lemma}

\ifFULL
We define  $\prob{PSD-SSE}(\eta, \delta)$ as the special case of
\cref{prob:sse} where the adjacency matrix of the graph is positive
semidefinite.
We now show that this special case is again equivalent to the general case.

\begin{theorem} \label{thm:psd-sse}
  For all $\eta, \delta > 0$, there is $\eta' > 0$ such that
  $\prob{SSE}(\eta', \delta)$ is polynomial-time reducible to
  $\prob{PSD-SSE}(\eta, \delta)$.
\end{theorem}

\begin{proof}
  Fix $\eta,\delta' > 0$.
  Thanks to \cref{claim:sse-m}, it suffices to reduce from \cref{prob:sse-m}.
  That is, we will show that there are $\eta' > 0$ and $M \geq 1$ such that
  $\prob{SSE}(\eta', \delta, M)$ is polynomial-time reducible to
  $\prob{PSD-SSE}(\eta, \delta)$.

\else
We define  $\prob{PSD-SSE}(\eta, \delta)$ as the special case of
\cref{prob:sse-m} where the adjacency matrix of the graph is positive
semidefinite and $M=1$.
We now show that this special case is again equivalent to the general case.

\begin{theorem} \label{thm:psd-sse}
  For all $\eta, \delta > 0$, there is $\eta' > 0$ such that
  $\prob{SSE}(\eta', \delta, M)$ is polynomial-time reducible to
  $\prob{PSD-SSE}(\eta, \delta)$ for $M \defas 4/\eta$.
\end{theorem}

\begin{proof}
\fi
  We will assume $\eta \leq 1/2$ (if \prob{PSD-SSE($\eta$, $\delta$)} is hard
  then so is the same problem with larger $\eta$).
  Let $t \defas 128\log(1/\eta)$, $\eta' \defas \min(\eta, 2\eta/t)$
\ifFULL 
, $M \defas 4/\eta$.
\fi

  The reduction takes an instance $G$ of \prob{SSE($\eta', \delta, M$)} and
  outputs $G_\lazy^t$.
  The lazy random walk matrix $\matG_\lazy$ is positive semidefinite, and hence
  so is $\matG_\lazy^t$.
  As a result, the output is an instance of \prob{PSD-SSE}.

  \textbf{Yes case:}
By \cite[Lemma~12]{RaghavendraS14}, for every subset $S$, $\Phi_{G_\lazy^t}(S) \leq t\Phi_{G}(S)/2$.
In particular, if $G$ is a Yes case of $\prob{SSE}(\eta', \delta, M)$,
then for some subset $S$ of size $\delta n$, has $\Phi_G(S) \leq \eta'$, and thus also $\Phi_{G_\lazy^t}(S) \leq \eta$.

\ignore{
  Notice that
  \begin{gather*}
  \Phi_{G_\lazy}(S) = \Phi_{(G+I)/2}(S) = \left(\Phi_G(S) + \Phi_I(S) \right) /2 =  \Phi_G(S)/2.
  \end{gather*}
  
  We claim that $ \Phi_{G_\lazy^t}(S) \leq t\Phi_{G_\lazy}(S)$.
  Assuming this claim, we have $\Phi_{G_\lazy^t}(S) \leq t\eta'/2 \leq \eta$.
  Since $\mu(S) = \delta$, $G_\lazy^t$ is a Yes instance of
  $\prob{PSD-SSE}(\eta, \delta)$.

  It remains to verify the claim.
  To show $\Phi_{G_\lazy^t}(S) \leq t\Phi_{G_\lazy}(S)$, note that
  \[ \Pr_{x\sim y\in G_\lazy^t}[x\in S, y\notin S] = \Pr_{x_0\sim x_1\in
  G_\lazy, \dots, x_{t-1}\sim x_t\in G_\lazy}[x_0\in S, x_t\notin S] . \]
  When $x_0\in S \wedge x_t\notin S$, necessarily for some $i\in \domain t, x_i
  \in S \wedge x_{i+1}\notin S$.
  For any $i\in \domain t$, the event $x_i\in S \wedge x_{i+1}\notin S$ happens
  with probability $\Pr_{x\sim y\in G_\lazy}[x\in S, y\notin S]$.
  Taking a union bound over all $i\in \domain t$, and dividing by $\card S$, we
  get the claimed inequality.
}
  \textbf{No case:} This follows from \cref{lemma:power-expand}.
  Indeed, $\Phi_G(4\delta/\eta) \geq 1-\eta' \geq 1-\eta \geq 1/2$ by
  assumptions.
  Thus
  \[ \Paren{1-\frac{\Phi_G^2(4\delta/\eta)}{32}}^t \leq \Paren{1-\frac1{128}}^t
  \leq \exp(-t/128) \leq \eta . \]
  By \cref{lemma:power-expand}, $\Phi_{G_\lazy^t}(\delta) \geq 1-\eta$, and
  $G_\lazy^t$ is a No instance of $\prob{PSD-SSE}(\eta, \delta)$.
\end{proof}

Let us mention that a variant of the previous lemma follows from the techniques
of \cite{ChanKL15,KwokL14}, and in fact without making the graph lazy at all.

Given a PSD matrix $\matA$ of size $n$, let us define the sparse PCA objective
$\val_\matA(\delta) \defas \max_{\norm \vecx_2 = 1, \norm \vecx_0 \leq \delta
n} \vecx ^{\intercal} \matA \vecx$.

We also need the local version of Cheeger--Alon--Milman inequality
\cite[Theorem~1.7]{NatarajanW14}.

\begin{lemma} \label{lemma:cheeger}
  Let $\matL = \matI - \matG$ be the normalized Laplacian matrix of a regular
  graph $G$ on $n$ vertex.
  For any $\delta \leq 1/2$, let $\lambda_\delta =
  \min\set{\vecx^{\intercal}\matL\vecx/\vecx^{\intercal}\vecx \mid \norm \vecx_0 \leq \delta n}$.
  Then
  \[ \Phi_G(\delta) \leq \sqrt{(2-\lambda_\delta)\lambda_\delta} . \]
\end{lemma}

\begin{theorem} \label{thm:sse-spca}
  If $G$ is a Yes instance of \prob{PSD-SSE}$(\eta, \delta)$, then
  $\val_\matG(\delta) \geq 1-\eta$.
  If $G$ is a No instance of \prob{PSD-SSE}$(\eta, \delta)$, then
  $\val_\matG(\delta) \leq \sqrt{1-(1-\eta)^2}$.
\end{theorem}

\begin{proof}
  \textbf{Yes case:}
  Let $S$ be a subset with $\card S \leq \delta n$ and $\Phi_G(S) \leq \eta$.
  Consider the normalized indicator function $\1_S:V\to \R$ for $S$.
  $\1_S$ has at most $\delta n$ non-zero entries, and by normalization, $\norm {\1_S}_2 = 1$. Furthermore,
  \[ \1_S^{\intercal}\matG \1_S = \frac{\sum_{i,j\in S}
    \matG_{ij}}{\card S} = \frac{\sum_{i\in S} (1-\sum_{j\notin S}
  \matG_{ij})}{\card S} = 1 - \frac{\sum_{i\in S, j\notin S} \matG_{ij}}{\card
  S} = 1-\Phi_G(S) . \]
  Therefore $\val_\matG(\delta) \geq 1-\eta$.

  \textbf{No case:}
  Let $\vecx$ be any $\delta n$-sparse vector.
  Then
  \[ \frac{\vecx^{\intercal}\matG\vecx}{\vecx^{\intercal}\vecx} =
    1-\frac{\vecx^{\intercal}\matL\vecx}{\vecx^{\intercal}\vecx} \leq 1-\lambda_\delta,
  \]
  where $\lambda_\delta$ is as defined in \cref{lemma:cheeger} and satisfies
  \[ \sqrt{(2-\lambda_\delta)\lambda_\delta} \geq \Phi_G(\delta) \geq 1-\eta.
  \]
  Letting $\rho \defas 1-\lambda_\delta$, the previous inequality becomes
  $1-\rho^2 = (1+\rho)(1-\rho) \geq (1-\eta)^2$, and hence
  $\vecx^{\intercal}\matG\vecx/\vecx^{\intercal}\vecx \leq \rho \leq \sqrt{1-(1-\eta)^2}$.
\end{proof}

\cref{thm:psd-sse} implies \SPCA{} is hard to solve within any constant factor
$C$.
Indeed, let $\eta \defas \min(1-\sqrt{1-1/4C^2}, 1/2)$.
\cref{thm:psd-sse,thm:sse-spca,conj:sseh} imply that given the matrix $\matG$
in the output of \cref{thm:sse-spca}, it is \NP-hard to tell whether
$\val_\matG(\delta) \geq 1-\eta \geq 1/2$, or $\val_\matG(\delta) \leq
\sqrt{1-(1-\eta)^2} = 1/2C$.

%% file: rank-gap.tex

\section{SDP gap} \label{sec:gap}

Recall the SDP for sparse PCA proposed by \cite{d'AspremontEGJL07}:
\ifFULL
\begin{equation} \label{eq:sdp}
  \begin{split}
    \max \quad & \tr(\matA \matX) \\
    \text{such that} \quad & \tr(\matX) = 1 \\
                           & \1^{\intercal} \abs \matX\1 \leq k \\
                           & X \geqsd 0
  \end{split}
\end{equation}
\fi
In this section, we will show that the SDP has a factor
$\exp\exp(\Omega(\sqrt{\log\log n}))$ gap.

If $\matA$ is the adjacency matrix of a graph, then the SDP is essentially
identical to the SDP for small-set expansion in \cite{RaghavendraST10}.
Gap instances for the latter problem therefore imply strong rank gap for sparse
PCA, provided the adjacency matrix is PSD.
A typical gap instance for small-set expansion SDP is the noisy hypercube of
dimension $\log n$ with $n$ vertices.
It is not hard to see that its adjacency matrix leads to $(\log n)^{\Omega(1)}$
gap for sparse PCA SDP.
Below we use a more sophisticated graph $G$ that can be considered as a small
induced subgraph of the noisy hypercube (even though formally $G$ is not such a
subgraph).
This will lead to $\exp\exp(\Omega(\sqrt{\log \log n}))$ gap for sparse PCA
SDP, where $n$ is the number of vertices in this graph.
This gap factor is super-polylogarithmic but sub-polynomial.

\subsubsection*{Construction}
The gap instance $\matA$ for the SDP is derived from the short code graph $G$ from
\cite{BarakGHMRS12}, also known as the low-degree long code.
Its vertex set is the Reed--Muller code $\RM(m,d)$ (evaluations of polynomials
of (total) degree $\leq d$ over $\F_2$ in $m$ variables $x_1, \dots, x_m$).
Two vertices are connected if their corresponding polynomials differ by a
product of exactly $d$ linearly independent affine forms.
Call $T$ the collection of all such affine forms.
Therefore $G$ is the Cayley graph on $\RM(m,d)$ with generating set $T$.

The matrix $\matA$ will be the adjacency matrix for continuous-time random walk on
$G$.
That is, $\matA = e^{-t(\matI-\matG)}$ for some $t \geq 0$.
Here we denote by $\matG$ the probability transition
matrix for the graph $G$.
Therefore $\matG$ is a matrix where every row and every column sum to $1$.
As in \cite{BarakGHMRS12}, taking a continuous-time random walk significantly
reduces the value of the quadratic form for sparse vectors.
For our application, continuous-time random walk has the additional benefit
that $\matA$ is guaranteed to be PSD because $\matA$ is the exponentiation of a real
symmetric matrix.

It will be more convenient to transform \cref{eq:sdp} into the following SDP:
\begin{equation} \label{eq:sdp2}
  \begin{split}
    \max \quad & \E_f \inner{\vecw_f, (\matA \vecw)_f} \\
    \text{such that} \quad & \E_f \inner{\vecw_f, \vecw_f} = 1 \\
                           & \E_{f,g} \abs{\inner{\vecw_f, \vecw_g}} \leq \delta = k/n
  \end{split}
\end{equation}
The SDPs in \cref{eq:sdp,eq:sdp2} are indeed equivalent, because any SDP
solution $\matX$ to \cref{eq:sdp} is the (scaled) Gram matrix
\begin{equation} \label{eq:gram}
  \matX_{f,g} = \inner{\vecw_f, \vecw_g}/n,
\end{equation}
of some vectors $\vecw_f\in \R^n$, and vice versa.

\textbf{Choice of parameters:}
$m$ is a free parameter that all other parameters depend on.
Let $\delta \defas 1/2^{m/2}$ be the fractional sparsity parameter.
Let $\eta \defas \delta^{1/(4\log 3)}$ be the eigenvalue threshold. 
Let $\eps_2 = \min\{\eps_1, 1/20\}$, where $\eps_1$ is the
constant from \cite[Theorem~1]{BhattacharyyaKSSZ10}. 
Let $d \defas \log\log(1/\eta) + \log(1/\eps_2) - 1$ be the degree of the Reed Muller code, and let $t \defas 2^{d-1}$ be the time parameter for the continuous random walk.
Let $n \defas \card{\RM(m,d)}= 2^{m \choose \leq d}$ be the size of $\matA$.
Here ${m \choose \leq d} \defas \sum_{r\leq d} {m \choose r}$ denotes the
number of ways to choose a subset of size $\leq d$ out of $m$ elements.
Let $k \defas n/2^{m/2}$.

\begin{proposition} \label{prop:complete}
  The SDP in \cref{eq:sdp2} has a solution of value $1/e = \Omega(1)$.
\end{proposition}

\begin{proof}
Let $\vecw_f$ by the standard embedding of $f\in \RM(m,d)$.
That is, $\vecw_f:\F_2^m \to \R$ is the vector/function such that its
$x$-coordinate is $\vecw_f(x) = (-1)^{f(x)} \in \set{\pm 1}$ for $x\in \F_2^m$.
This defines a solution to \cref{eq:gram}.
In \cref{eq:gram} and below, the inner product $\inner{\cdot, \cdot}$ on
$\F_2^m\to \R$ is defined as $\inner{\vecw,\vecw'} \defas \E_{x\in \F_2^m} \vecw(x)\vecw'(x)$.

\ifFULL

We now verify that $\matX$ is a feasible solution to the SDP.
As a Gram matrix, $\matX$ is clearly PSD.
Also
\[ \E_{f\in \RM(m,d)} \inner{\vecw_f,\vecw_f} = \E_{f\in \RM(m,d)} \E_{x\in \F_2^m}
[((-1)^{f(x)})^2] = 1, \]
and
\begin{equation} \label{eq:global-corr}
  \E_{f,g\in \RM(m,d)} \abs{\inner{\vecw_f,\vecw_g}} = \E_{f,g\in \RM(m,d)}
  \Abs{\E_{x\in \F_2^m} (-1)^{f(x)-g(x)}} = \E_{h\in \RM(m,d)} \Abs{\E_{x\in
  \F_2^m} (-1)^{h(x)}},
\end{equation}
where in the last equality we let $h = f - g$.
Using Cauchy--Schwarz, the right-hand-side is at most
\begin{equation} \label{eq:global-corr2}
  \sqrt{\E_h \Paren{\E_{x\in F_2^m} (-1)^{h(x)}}^2} = \sqrt{\E_{x,y\in F_2^m}
  \E_h (-1)^{h(x)-h(y)}} .
\end{equation}
We now analyze the term inside the square root.
When $x \neq y$,
\[ \E_h (-1)^{h(x)-h(y)} = 0 , \]
thanks to pairwise independence of $\RM(m,d)$.
When $x = y$ (which happens with probability $1/2^m$), the same expectation is
$1$.
Therefore \cref{eq:global-corr2} is at most $1/2^{m/2}$, and so is
\cref{eq:global-corr}.
Then $\matX$ satisfies the sparsity constraint with $k/n = 1/2^{m/2}$.
\fi

We now bound the SDP value.
Let $\phi_x(f) \defas (-1)^{f(x)}$.
Then
\[ \E_f \inner{\vecw_f, (\matA \vecw)_f} = \E_f \E_{x\in \F_2^m} (\phi_x)(f)(\matA\phi_x)(f) .
\]
We claim that $\phi_x$ is an eigenfunction of $\matA$ with eigenvalue $1/e$.
Assuming this claim, the right-hand side becomes
\[ (1/e) \cdot  \E_f \E_{x\in \F_2^m} [\left(\phi_x(f)\right)^2] = 1/e , \]
giving an SDP solution of value $1/e$.

We now verify the claim.
For every $x\in \F_2^m$, the function $\phi_x(f) = (-1)^{f(x)}$ is an
eigenvector of $\matG$ because
\[ (\matG\phi_x)(f) = \E_{g\in T} (-1)^{f(x)-g(x)} = \phi_x(f) \cdot \E_{g\in T}
(-1)^{g(x)} . \]
It has eigenvalue
\[ \lambda_x \defas \E_{g\in T} (-1)^{g(x)} = 1 - 2\Pr_{g\in T}[g(x) = 1] =
1-2^{1-d} . \]
Since $\matG$ and $\matA$ have the same eigenvectors, $\phi_x$ is also an eigenvector
of $\matA$ with eigenvalue
\[ e^{-t(1-\lambda_x)} = e^{-t2^{1-d}} = 1/e . \qedhere \]
\end{proof}

\begin{proposition} \label{prop:sound}
  Any $k$-sparse rank-$1$ solution $\vecw:\RM(m,d) \to \R$ to \cref{eq:sdp2} has
  value $\leq \eta + (1/\eta)^{\log 3} \sqrt{k/n}$. 
\end{proposition}

Since the proof is quite technical, let us recall main ideas in
\cite{BarakGHMRS12}.
Intuitively, the sparse PCA instance $\matA$ has low value for rank-$1$ sparse
vector for the following reason.
The inner product space $V(G) \to \R$ can be decomposed into a sum of the
subspace $V_\ell$ and its orthogonal complement $V_\ell^\perp$.
One can show that $V_\ell$ does not contain any sparse vector (more precisely,
has bounded $2$-to-$4$ norm).
Therefore any sparse vector must be essentially contained in (i.e.~has large
projection to) $V_\ell^\perp$.
$V_\ell^\perp$ will be the span of eigenvectors of $\matA$ whose eigenvalues are
small, say at most a small positive number $\eta$.
This ensures all sparse vectors have small objective value under the quadratic
form, as desired.

\begin{proof}
  This is essentially Theorem~4.14 in \cite{BarakGHMRS12}.
  Even though their statement only concerns $\set{0,1}$-valued sparse vectors,
  their proof also works for real-valued sparse vectors
\ifFULL
, as we now show.

\subsubsection*{Setting up the Fourier expansion}

  Let $M \defas 2^m$ and $C = \RM(m,d)$.
  We first think of the elements of $C$ as functions $\F_2^m\to \F_2$;
  later it will be more convenient to think of them as vectors in $F_2^M$. 
  For $c_1, c_2 : \F_2^m\to \F_2$ denote the inner product
  $(c_1, c_2)_2 \defas  \sum_{x \in \F_2^m} c_1(x) c_2(x) \pmod 2$

  Denote by $C^\perp \defas \set{a\in \F_2^M\mid (a,c)_2 = 0 \;\text{for all
  $c\in C$}}$ the orthogonal subspace of $C$.

  Any function $\vecw:C\to \R$ has a Fourier expansion, as follows.
  For every coset $\alpha+C^\perp\in \F_2^M/C^\perp$, we choose an arbitrary
  representative $\alpha$ in $\alpha+C^\perp$, and let $\chi_\alpha(f) =
  (-1)^{(\alpha,f)_2}$ be its character.
  Its degree is $\deg_{\R}(\chi_\alpha) \defas \min_{c^\perp\in C^\perp} \abs{\alpha + c^\perp}$,
  where $\abs \alpha$ denotes the Hamming weight (i.e.~number of non-zero coordinates)
  of $\alpha$. (Do not confuse this degree with the degrees of polynomials in the Reed Muller code!)
  Any function $\vecw:C\to \R$ is a unique linear combination of characters
  $\set{\chi_\alpha}_{\alpha\in \F_2^M/C^\perp}$,
  \[ \vecw(f) = \sum_{\alpha\in \F_2^M/C^\perp} \hat \vecw(\alpha) \chi_\alpha(f) , \]
  where $\hat \vecw(\alpha) \defas \inner{\chi_\alpha, \vecw}$ is the Fourier transform
  of $\vecw$ over the abelian group $C$.

  Set the character degree bound $\ell \defas \eps_2 2^{d+1}$.
  Consider the subspace $V_\ell \defas \Span\set{\chi_\alpha \mid
  \deg_{\R}(\chi_\alpha) \leq \ell}$ of functions of degree at most $\ell$.
  Note that $V_\ell$ and $V_\ell^\perp$ are both invariant subspaces of $\matA$.

  Given any vector $\vecw$, we expand it as $\vecw = \vecw^\para + \vecw^\perp$ where $\vecw^\para\in V_\ell$ and
  $\vecw^\perp\in V_\ell^\perp$.
  Then
  \begin{equation} \label{eq:split}
    \inner{\vecw,\matA \vecw} = \inner{\vecw^\para,\matA \vecw^\para} + \inner{\vecw^\perp,\matA \vecw^\perp} .
  \end{equation}
 Below, we separately bound the contribution of $\inner{\vecw^\para,\matA \vecw^\para}$ and $\inner{\vecw^\perp,\matA \vecw^\perp}$.
 
\subsubsection*{The low-degree subspace $V_\ell$}

  Consider the projection operator $P_\ell$ to the subspace $V_\ell$.
  The $p$-to-$q$ norm of $P_\ell$ is defined as
  \[ \norm{P_\ell}_{p\to q} \defas \max_{\vecw:C\to \R} \frac{\norm{
  P_\ell\vecw}_q}{\norm \vecw_p} , \]
where in the case of a function $\vecw:C\to \R$, we define $\norm \vecw_p \defas \E_{x\in C} [\abs{\vecw(x)}^p]^{1/p}$.

  We use the following bound on the $2$-to-$4$ norm of $P_\ell$, from \cite[Lemma~4.9]{BarakGHMRS12}: 
  For any $\ell < (2^{d-1}-1)/4$,
  \begin{equation}  \label{eq:24norm}
    \norm{P_\ell}_{2\to 4} \leq 3^{\ell/2} .
  \end{equation}

  For any $k$-sparse vector $\vecw:C\to \R$, let $S = \set{x\in C\mid \vecw(x) \neq
  0}$ be the set of nonzero entries.
  By H\"older's inequality,
  \begin{equation} \label{eq:sparse}
    \norm \vecw_{4/3} = \norm{\1_S\cdot \vecw}_{4/3} \leq \norm{\1_S}_4 \norm \vecw_2 =
    (k/n)^{1/4}\norm \vecw_2.
  \end{equation}

Recall that $\matA = e^{-t(\matI-\matG)}$. Since $(\matI-\matG)$ is PSD, we have that all of $\matA$'s eigenvalues are at most $1$, i.e. $\matI \succcurlyeq \matA$. 
Therefore, 
  \[ \inner{\vecw^\para,\matA \vecw^\para} \leq \norm{\vecw^\para}_2^2 \leq
  \norm{P_\ell}_{4/3\to 2}^2 \norm{\vecw}_{4/3}^2 , \]

  Together with $\norm{P_\ell}_{4/3\to 2} \leq \norm{P_\ell}_{2\to 4}$
  \cite[Lemma~4.2]{BarakGHMRS12} and \cref{eq:24norm,eq:sparse}, we get
  \begin{equation} \label{eq:top}
    \inner{\vecw^\para,\matA \vecw^\para} \leq 3^\ell\sqrt{k/n}\norm \vecw_2^2 = (1/\eta)^{\log 3}
    \sqrt{k/n} \norm \vecw_2^2,
  \end{equation}

where the last equation follows from $3^\ell = 3^{\eps_2 2^{d+1}} = 3^{\log (1/\eta)}$.

\subsubsection*{The high-degree subspace $V_\ell^{\perp}$}

  We now bound the second term $\inner{\vecw^\perp,\matA \vecw^\perp}$.
  $\vecw^\perp$ is a linear combination of characters of degree $>\ell$.
  Recall that $T$, the generating set of $G$, is the set of products of exactly $d$ linearly independent affine forms.
  Any character $\chi_\alpha$ is an eigenvector of $\matG$ because
  \[ (\matG \chi_\alpha)(f) = \E_{g\in T} [\chi_\alpha(f+g)] = \E_{g\in
  T}[(-1)^{(\alpha,f+g)_2}] = (-1)^{(\alpha,f)_2} \E_{g\in T}[(-1)^{(\alpha,g)_2}] =
  \chi_\alpha(f) \E_{g\in T} \chi_\alpha(g) , \]
  and its eigenvalue is
  \[ \lambda_\alpha \defas \E_{g\in T} \chi_\alpha(g) = \E_{g\in T}
  [(-1)^{(\alpha,g)_2}] = 1-2\Pr_{g\in T}[(\alpha,g)_2 = 1] , \]
  
  We now use a theorem about Reed Muller code testers to bound $\Pr_{g\in T}[(\alpha,g)_2 = 1]$.
  An important problem in the intersection of coding theory and property testing is as follows:
  given a code $C^\perp$ and a word $\alpha$, query a small number of $\alpha$'s bits to decide whether $\alpha$ belongs to the code, or is far from the code.
  By ``far'' from the code, it is meant that it has a large Hamming distance from any $c^\perp \in C^\perp$.
  When $C^\perp$ is a Reed-Muller code, in particular $\RM(m,m-d-1)$, this is equivalent to testing whether $\alpha$ is a low ($m-d-1$) degree polynomial, or far from every low degree polynomial.
  A canonical test for this problem is as follows: pick a random $(m-d)$-dimensional affine subspace $S_g$, and test whether $\alpha$ restricted to this subset is a degree-$(m-d-1)$ polynomial.
  
  It turns out that having degree $\leq m-d-1$ over $S_g$ corresponds exactly to having $\sum_{x \in S_g} \alpha(x) = 0 \pmod 2$~\cite{BhattacharyyaKSSZ10}.
  (Proof sketch: any monomial of degree $\leq m-d-1$ does not contain at least one of the $m-d$ variables, and thus zeros out when we sum modulo $2$ over that variable;
  in the other direction, there is only one homogenous full-degree monomial, and it is nonzero only on the all-ones input.)
  
Furthermore, picking a random $(m-d)$-dimensional affine subspace $S_g$ corresponds precisely to picking a random $g \in T$ and letting $S_g \defas \{x: g(x) = 1\}$.
  (This is related to ``dual codes''; see also~\cite{AKKLR05-testing_RM}.)
  In other words, the test is the same as verifying that $(\alpha,g)_2 = 0$.

  Bhattacharyya et al.~\cite{BhattacharyyaKSSZ10} analyze the probability that the above test rejects polynomials that are far from the code, 
  i.e. precisely the quantity $\Pr_{g\in T}[(\alpha,g)_2 = 1]$.
Recall that the degree of $\chi_{\alpha}$ was defined as the Hamming distance of $\alpha$ from $C^\perp$.
  By our assumption that $\chi_\alpha \in V_\ell^\perp$, we have that $\deg_{\R}(\chi_\alpha) \geq \ell=\eps_2 2^{d+1}$;
  that is $\alpha$ disagrees with every $c^{\perp} \in C^{\perp}$ on at least $\eps_2 2^{d+1} / 2^m = \eps_2 2^{-(m-d-1)}$-fraction of the entries.
  Therefore, by \cite[Theorem~1]{BhattacharyyaKSSZ10}, we have that $\Pr_{g\in T}[(\alpha,g)_2 = 1] \geq \eps_2$.

  As a result, any $\chi_\alpha$ with $\deg_{\R}(\chi_\alpha) \geq \eps_2 2^{d+1}$ is
  also an eigenvector of $\matA$ with eigenvalue
  \[ \mu_\alpha \defas e^{-t(1-\lambda_\alpha)} \leq e^{-\eps_2 2^{d+1}} \leq \eta
  . \]
  Therefore
  \begin{equation} \label{eq:bottom}
    \inner{\vecw^\perp, \matA \vecw^\perp} = \mathop{\sum_{\alpha\in
    \F_2^M/C^\perp}}_{\deg_{\R}(\chi_\alpha) \geq \ell} \mu_\alpha \hat \vecw(\alpha)^2
    \leq \eta \mathop{\sum_{\alpha\in \F_2^M/C^\perp}}_{\deg_{\R}(\chi_\alpha) \geq
    \ell} \hat \vecw(\alpha)^2 = \eta\norm{\vecw^\perp}_2^2 \leq \eta \norm \vecw_2^2 .
  \end{equation}

   Finally, \cref{prop:sound} follows from \cref{eq:split,eq:top,eq:bottom} and the
  constraint $\norm \vecw_2^2 \leq 1$.

\fi
\end{proof}

We remark that an alternative proof of the previous proposition (with a
slightly different bound) can be obtained by combining Theorem~4.14 in
\cite{BarakGHMRS12} and local Cheeger--Alon--Milman inequality
\cite[Theorem~1.7]{NatarajanW14}.

\begin{theorem}
  Let $\matA$ be the matrix defined above.
  The SDP in \cref{eq:sdp2} has an SDP solution of value $\Omega(1)$, but any
  rank-$1$ solution has value $1/\exp\exp(\Omega(\sqrt{\log\log n}))$.
\end{theorem}

\begin{proof}
  The SDP solution is given in \cref{prop:complete}.
  On the other hand, \cref{prop:sound} shows that any rank-$1$ solution has
  value $(k/n)^{\Omega(1)}$.
  Since $\log n = {m \choose \leq d}$, we have $\log\log n = (\log
  m)^2(1+o_m(1))$ and $(k/n)^{\Omega(1)} = 1/\exp(\Omega(m)) =
  1/\exp\exp(\Omega(\sqrt{\log\log n}))$.
\end{proof}

%% file: aPTAS.tex

\section{Additive PTAS}
\label{sec:aPTAS}

To complete the approximability picture for \SPCA{}, we briefly sketch the proof of the additive \PTAS{}
due to \cite{asteris2015sparse}.
The algorithm first approximates $\matA$ with a low-rank sketch, 
and then finds approximate solutions via an $\epsilon$-net search of the low dimensional space.
(We note that a similar approach was previously presented in \cite{alon2013approximate}, for the closely related problem of \DkS{} on a PSD adjacency matrix.)

The existence of a low-rank sketch, due to Alon et al., is via an application of the Johnson-Lindenstrauss Lemma:
\begin{lemma}[\cite{alon2013approximate}]
Let $\matA \in \R^{n\times n}$ be PSD matrix with entries in $[-1,1]$.
Then, we can construct in polynomial time a PSD matrix $\matA_{\epsilon}$ with rank $O(\frac{\log n}{\epsilon^{2}})$ such that
\begin{align}
	\left|[\matA]_{i,j}-[ \matA_{\epsilon}]_{i,j}\right|\le \epsilon \nonumber
\end{align} for all $i,j$ with high probability.
\end{lemma}
The above low-rank approximation to $\matA$ preserves all $k$-sparse quadratic forms to within an additive error term:
\begin{align}
	|{\bf x}^{\intercal} \matA{\bf x}-{\bf x}^{\intercal}{\matA_{\epsilon}}{\bf x}| &= \left| \sum_{i,j} \vecx_i \vecx_j ([\matA]_{ij}-[{\matA_{\epsilon}}]_{ij}) \right|  \le \epsilon \left|\sum_{i=1}^n |\vecx_i| \sum_{j=1}^n |\vecx_j| \right|  =\epsilon\|\vecx\|_1^2\le \epsilon k
	\label{eq:A-Ad-quadratic}.
\end{align}

Since $\matA$ is PSD, one can rewrite $\matA = \matB^{\intercal}\matB$, where $\matA$'s low-rank property translates to $\matB$ having few columns.
Enumerating over an $\epsilon$-net on the low dimension of $\matB$ now gives, results in the following:
\begin{lemma}[\cite{asteris2015sparse}]
Let $\matA_d \in \R^{n\times n}$ be PSD matrix of rank $d$.
Then, we can construct a vector $\vecx_d$, in time $O(\epsilon^{-d} \cdot n \log n)$, such that 
$$\vecx_d^{\intercal}\matA_d \vecx_d > (1-\epsilon)\cdot OPT.$$
\end{lemma}
Finally, combining the above two results gives the additive \PTAS{}.
\begin{theorem}[\cite{asteris2015sparse}]
Let $\matA \in \R^{n\times n}$ be PSD matrix with entries in $[-1,1]$.
Then, we can compute in  $n^{O(\poly(1/\epsilon))}$ time a $k$-sparse unit norm vector $\vecx_{\epsilon}$ such that
$$\vecx_{\epsilon}^{\intercal}\matA \vecx_{\epsilon} \ge OPT-\epsilon\cdot k$$
with high probability.
\end{theorem}

%% file: non-PSD.tex

\section{When the input matrix is not PSD}\label{sec:non-PSD}
In this section, we briefly remark that although the \SPCA{} optimization problem can be defined when $\matA$ is not required to be PSD, no meaningful multiplicative approximation guarantees are possible (in polynomial time, assuming $\P \neq \NP$).

\begin{theorem}
When $\matA$ is not positive semi-definite, 
it is \NP-hard to decide whether the \SPCA{} objective is positive or negative.
\end{theorem}
\begin{proof}
Let $\val_\matA(k) \defas \max_{\norm \vecx_2 = 1, \norm \vecx_0 \leq k} \vecx ^T \matA \vecx$.
It is well known that solving the \SPCA{} exactly is \NP-hard even in the PSD case; i.e.~it is \NP-hard to distinguish between 
$\val_\matA(k)\geq c$ and $\val_\matA(k)\leq s$ for some (potentially very close) $c < s$. 
Consider the modified matrix $\matA' = \matA - \left(\frac{c+s}{2}\right)\cdot \matI$. Conclude that it is \NP-hard to distinguish
$\val_{\matA'}(k) \geq \frac{c-s}{2}$ and $\val_{\matA'}(k) \leq \frac{s-c}{2}$.
\end{proof}